\documentclass{article} 

\usepackage[accepted]{icml2019}

\usepackage{times}
\usepackage{multirow}
\usepackage{url}
\usepackage{qiangstyle}
\usepackage{mathrsfs}

\usepackage{booktabs}

\usepackage{lipsum}  

\icmltitlerunning{Improving Neural Language Modeling via Adversarial Training}

\newcommand{\softmax}{{\mathrm{Softmax}}}
\newcommand{\ours}{{\mathrm{AdvSoft}}}
\newcommand{\advsoft}{\ours}

\begin{document}

\twocolumn[
\icmltitle{Improving Neural Language Modeling via Adversarial Training}




\icmlsetsymbol{equal}{*}

\begin{icmlauthorlist}
\icmlauthor{Dilin Wang}{equal,ut}
\icmlauthor{Chengyue Gong}{equal,ut}
\icmlauthor{Qiang Liu}{ut}
\end{icmlauthorlist}
\icmlaffiliation{ut}{Department of Computer Science, UT Austin}
\icmlcorrespondingauthor{Dilin Wang}{dilin@cs.utexas.edu}
\icmlcorrespondingauthor{Chengyue Gong}{cygong@cs.utexas.edu}
\vskip 0.2in
]

\printAffiliationsAndNotice{\icmlEqualContribution}

\begin{abstract}
Recently, substantial progress has been made in language modeling by using deep neural networks. 
However, in practice, 
large scale neural language models have been shown to be prone to overfitting.
In this paper, 
we present a simple yet highly effective   
adversarial training mechanism for regularizing neural language models.
The idea is to introduce adversarial noise to the output embedding layer while training the models. We show that the optimal adversarial noise yields a simple closed form solution, thus allowing us to develop a simple and time efficient algorithm. 
Theoretically, we show that our adversarial mechanism effectively encourages the diversity of the embedding vectors, helping to increase the robustness of models. 
Empirically, we show that our method improves on the single model state-of-the-art results for language modeling on Penn Treebank (PTB) and Wikitext-2,  
achieving test perplexity scores of 46.01 and 38.07, respectively. 
When applied to machine translation, 
our method improves over 
various transformer-based translation baselines 
in BLEU scores on the WMT14 English-German and IWSLT14 German-English tasks.
\end{abstract}
\vspace{-2em}


\section{Introduction}

Statistical language modeling is a 
fundamental task in machine learning, with wide applications in various areas, 
including automatic speech  recognition  \citep[e.g.,][]{yu2016automatic}, machine translation \citep[e.g.,][]{koehn2009statistical} and computer vision \citep[e.g.,][]{xu2015show},
to name a few. 
Recently, deep neural network models, especially
recurrent neural networks (RNN) based models,  
have emerged to be one of the most powerful 
approaches for language modeling  
\citep[e.g.,][]{merity2017regularizing, yang2017breaking, vaswani2017attention, anderson2018bottom}. 

Unfortunately, a major challenge in  training large scale RNN-based language models is  their tendency to overfit; 
this is caused by the 
high complexity of RNN models 
and the discrete nature of language inputs.  
Although various regularization techniques, such as early stop and  dropout \citep[e.g.,][]{gal2016theoretically}, 
have been investigated, severe overfitting is still widely observed in state-of-the-art benchmarks, as evidenced by the large gap between training and testing performance. 


In this paper, we develop
a simple yet surprisingly efficient minimax training strategy for regularization.
Our idea is to inject an adversarial perturbation  
on the word embedding vectors in the softmax layer of the language models,     
and seek to find the optimal parameters that  maximize the worst-case performance subject to 
the adversarial perturbation.   
%
%
Importantly, we show that the optimal perturbation vectors yield a simple and computationally efficient form under our construction, allowing us to derive a simple 
and fast training algorithm (see Algorithm~\ref{alg:main}),  
which can be easily implemented based a minor modification of the standard maximum likelihood  training and does not introduce additional training parameters.  

An intriguing theoretical property of our method is that it provides an 
effective mechanism to encourage
diversity of word embedding vectors,  
which is widely observed to yield 
better generalization performance  
in neural language models \citep[e.g.,][]{
mu2017all, gao2018representation, 
liu2018learning, cogswell2015reducing, khodak2018carte}. 
%
In previous works, the diversity is often enforced explicitly by adding additional diversity penalty terms \citep[e.g.,][]{gao2018representation}, 
which may impact the likelihood optimization and are computationally expensive when the vocabulary size is large. 
Interestingly, we show that our adversarial training  
effectively enforces diversity without explicitly introducing the additional diversity penalty,  and is significantly more  computationally efficient than direct regularizations. 

Empirically, 
we find that our adversarial method 
can significantly improve the performance of 
state-of-the-art large-scale neural language modeling  and machine translation. 
For language modeling, 
we establish a new single model state-of-the-art result
for the Penn Treebank (PTB) and WikiText-2 (WT2) datasets to the best of our knowledge, achieving 46.01 and 38.07 test perplexity scores, respectively.  
On the large scale WikiText-103 (WT3) dataset, our method improves the Quasi-recurrent neural networks (QRNNs) \citep{merity2018analysis} 
baseline. 

To demonstrate the broad applicability of the method, 
we also apply our method to improve machine translation, using 
Transformer \citep{vaswani2017attention} as our base model. 
By incorporating our adversarial training,  
we improve a variety of Transformer-based translation baselines on the WMT2014 English-German and IWSTL2014 German-English translations.




\section{Background: Neural Language Modeling}
Typical word-level language models are specified as a product of conditional probabilities using the chain rule: 
\begin{align}\label{cond}
p(x_{1:T}) = \prod_{t=1}^T p(x_t~|~x_{1:t-1}),
\end{align}
where $x_{1:T} = [x_1, \cdots, x_T]$ denotes a sentence of length $T$, with 
$x_t \in \vb$ the $t$-th word and $\vb$ the 
vocabulary set. 
In modern deep language models, the conditional probabilities $p(x_t|x_{1:t-1})$  are often specified using recurrent neural networks (RNNs), 
in which  the context $x_{1:t-1}$ 
at each time $t$ is represented 
using a hidden state vector $h_t \in \RR^{d_h}$ defined recursively via 
\begin{align}
     h_t = f(x_{t-1}, h_{t-1}; \vv\theta),  \label{equ:fht}
\end{align}
where $f$ is a nonlinear map with a trainable parameter $\vv \theta$. 
The conditional probabilities are then defined using a softmax function: 
\begin{align} 
\begin{split}
 p(x_t ~|~x_{1:t-1}; \vv\theta, \vv w) 
& = \softmax(x_{t}, \vv w, h_t) \\
& :=  \frac{\exp(w_{x_t}^\top h_t)}{\sum_{\ell =1}^{|\mathcal V|} \exp(w_{\ell}^\top h_t)},
\end{split}
 \label{equ:softmax}
\end{align}
where $\vv w = \{w_i\}\subset \RR^{d}$ is the coefficient of softmax;  
$w_i$ can be viewed as an embedding vector for word $i \in \mathcal V$
and $h_t$ the embedding vector of context $x_{1:t-1}$.  
The inner product $w_{x_t}^\top h_t$ measures the similarity between word $x_t$ and context $x_{1:t-1}$, which is converted into a probability using the softmax function. 

In practice, the nonlinear map $f$ is specified by typical RNN units, 
such as LSTM \citep{hochreiter1997long} or GRU \citep{chung2014empirical}, 
applied on another set of embedding vectors $w_i'\in \RR^{d'}$ of the words, 
that is, 
$$
f(x_{t-1}, h_{t-1}; ~\vv\theta) = f_{RNN}(w'_{x_{t-1}}, h_{t-1}; ~\vv\theta'), 
$$
where $\vv\theta'$ is the weight of the RNN unit $f_{RNN}$, 
and  $\vv\theta = [\vv w', \vv\theta']$ is trained jointly with $\vv w$. 
Here, $w'_i$ is the embedding vector of word , fed into the model from the input side (and hence called the \emph{input embedding}), while $w_i$ is the embedding vector from the output side (called the \emph{output embedding}). 
It has been found that it is often useful to tie the input and output embeddings, that is, setting $w_i = w_i'$ (known as the weight-tying trick), 
which reduces the total number of free parameters and yields significant improvement of performance   \citep[e.g.,][]{press2016using, inan2016tying}. 


Given a set of sentences $
\{x_{1:T}^\ell\}_{\ell}$, 
the parameters $\vv\theta$ and $\vv w$ are jointly trained by maximizing the log-likelihood:  
\begin{align} \label{equ:mle}
\max_{\vv\theta, \vv w} \left \{ \mathcal{L}(\vv \theta, \vv w) := \sum_{t,\ell} \log p(x_t^\ell\mid x_{1:t-1}^\ell; \vv \theta, \vv w) \right\}. 
\end{align}
This optimization involves joint training of 
a large number of parameters $[\vv\theta,\vv w]$, including both the neural weights and word embedding vectors,  and is hence highly prone to overfitting in practice. 

\section{Main Method} 
\label{sec:main}
We propose a simple 
algorithm that 
effectively  
alleviates overfitting in deep neural language models, 
based on injecting adversarial perturbation on the output embedding vectors $w_i$ in  the softmax function (Eqn. \eqref{equ:softmax}).  
Our method is \emph{embarrassingly simple}, adding virtually no additional computational overhead over standard maximum likelihood training, 
while achieving substantial improvement on challenging benchmarks (see Section~\ref{sec:experiement}). 
We also draw theoretical insights on this simple mechanism,  
showing that it implicitly 
promotes diversity among the output embedding vectors $\{w_i\}$,
which is widely believed to increase robustness of the results \citep[e.g.,][]{cortes1995support, liu2018learning, gao2018representation}. 


\subsection{Adversarial MLE} 

Our idea is to introduce an adversarial noise 
on the output embedding vectors $\vv w = \{w_i\}$ in maximum likelihood training \eqref{equ:mle}: 
\begin{align}\label{equ:advsoft0}
\begin{split}
& \max_{\vv \theta,\vv w} \min_{\{\delta_{j;t,\ell}\}} \sum_{t,\ell} \log p(x_{t}^\ell ~|~ x_{1:t-1}^\ell; ~\vv\theta, ~\{ w_j+\delta_{j; t,\ell}\}), \\
& s.t. ~~~~ ||\delta_{j; t,\ell} ||\leq \epsilon/2, ~~~\forall j,t,\ell,
\end{split}
\end{align}
where $\delta_{j; t,\ell}$ is an adversarial perturbation applied on 
the embedding vector $w_j$ of word $j\in \mathcal V$, 
in the $\ell$-th sentence at the $t$-th location.  
We use $||\cdot||$ to denote the L2 norm throughout this paper; 
$\epsilon$ controls the magnitude of the adversarial perturbation.   


A key property of this formulation is that, with fixed model parameters $[\vv\theta,\vv w]$, 
the adversarial perturbation $\vv\delta = \{\delta_{i;t,\ell}\}$ has an elementary closed form solution, 
which allows us to derive a simple and efficient algorithm (Algorithm~\ref{alg:main}) by optimizing $[\vv\theta,\vv w]$ and $\vv\delta$ alternately.

\begin{thm}\label{thm:advh}
For each conditional probability term
$p(x_t=i ~|~x_{1:t-1}; \vv\theta, \vv w) 
= \softmax(i, \vv w, h_t)$  in \eqref{equ:softmax},
the optimization of the adversarial perturbation in \eqref{equ:advsoft0} is formulated as  
$$
\min_{\{\delta_j\}_{j\in \mathcal V}}  \frac{\displaystyle\exp((w_i + \delta_i)^\top h)}{\displaystyle\sum_j \exp((w_j + \delta_j)^\top h)}~~~~s.t~~~~ ||\delta_j||\leq \epsilon/2,~~~\forall j\in\mathcal V. 
$$
This is equivalent to just adding adversarial perturbation on $w_i$  
with magnitude $\epsilon$:
$$
\min_{\delta_i}  \frac{\displaystyle\exp((w_i + \delta_i)^\top h)}{\displaystyle
\exp((w_i + \delta_i)^\top h) + \sum_{j\neq i} \exp(w_j^\top h)}~~~~s.t~~~~ ||\delta_i||\leq \epsilon, 
$$
which is further equivalent to 
\begin{align}\label{equ:deltaistar}
\delta_i^* = \argmin_{||\delta_i||
\leq \epsilon }  (w_i + \delta_i)^\top h = - \epsilon h/||h||. 
\end{align}
As a result, we have 
\begin{align}\label{equ:advsoft} 
\begin{split}
& {\advsoft}_\epsilon(i, \vv w,h) \\
&~~~~~~~~~~~~~~  :=  \min_{||\delta_i||_2\leq\epsilon} \softmax(i, \{w_i+\delta_i, w_{\neg i}\}, ~ h) \\
&~~~~~~~~~~~~~~   = \frac{\displaystyle\exp(w_i^\top h-\epsilon||h||)}{\displaystyle
\exp(w_i^\top h-\epsilon||h||) + \sum_{j\neq i} \exp(w_j^\top h)},
\end{split}
\end{align}
where $w_{\neg i} = 
\{w_j \colon  j\neq i\}$. 
\end{thm}

In practice, we propose to optimize $[\vv\theta, \vv w]$ and $\vv \delta = \{\delta_{i;t,\ell}\}$ alternatively.  %
Fixing $\vv\delta$, the models parameters $[\vv\theta,\vv w]$  
are updated using  gradient descent as standard maximum likelihood training. 
Fixing $[\vv\theta, \vv w]$, 
the adversarial noise $\vv \delta$ is updated using the elementary solution in 
\eqref{equ:deltaistar}, 
which introduces almost no additional computational cost. 
See Algorithm~\ref{alg:main}. 
%
%
Our algorithm can be viewed as an approximate gradient descent optimization of    
$\mathrm{AdvSoft}_\epsilon(i,\vv w,h)$, but without back-propagating through the norm term $\epsilon ||h||$. Empirically, we note that back-propagating through $\epsilon ||h||$ 
seems to make the performance worse, as
the training error would diverge within a few epochs.
This is maybe because the gradient of $\epsilon ||h||$ forces $||h||$ to be large in order to increase $\mathrm{AdvSoft}_\epsilon(i, \vv w,h)$, which is not encouraged in our setting.

\subsection{Diversity of Embedding Vectors} 

An interesting property of our adversarial strategy is that  
it can be viewed as a mechanism to encourage diversity among word embedding vectors: 
we show that an embedding vector $w_i$ is guaranteed 
to be separated from the embedding vectors of all the other words by at least distance $\epsilon$,
once there exists a context vector $h$ 
with which $w_i$ dominates the other words according to $\advsoft$.  
This is a simple property implied by the definition of the adversarial setting:  
if there exists an $w_j$ within the $\epsilon$-ball of $w_i$, 
then $w_i$ (and $w_j$) can never dominate the other, 
because the winner is always penalized by the adversarial perturbation.   


\begin{mydef}\label{def:word}
Given a set of embedding vectors $\vv w = \{w_i\}_{i\in \mathcal V}$, 
a word $i\in \mathcal V$ is said to be \emph{$\epsilon$-recognizable} 
if there exists a vector $h \in \RR^d$  
on which $w_i$ dominates 
all the other words under  $\epsilon$-adversarial perturbation, in that 
\begin{align*} 
\min_{||\delta_i||\leq \epsilon} (w_i + \delta_i)^\top h
& = (w_i^\top h - \epsilon ||h||) \\
& > w_j^\top h, ~~~~  \forall  j\in \mathcal V, j\neq i.
\end{align*}
In this case, 
we have $\advsoft_\epsilon(i, \vv w, h)\geq 1/|\vb|$, and 
$w_i$ would be classified to be the target word of context $h$, despite the adversarial perturbation. 
\end{mydef}

\begin{algorithm}[t] 
\begin{algorithmic} 
    \STATE \textbf{Input} Training data $\mathcal D=\{x_{1:T}^\ell\}$, model parameters $\vv \theta, \vv w$
    \WHILE {not converge} 
    \STATE Sample a mini-batch $\mathcal{M}$ from the data $\mathcal D.$ 
    \STATE For each sentence $x_{1:T}^\ell$ in the minibatch and $t \leq T$, 
    set the adversarial noise on $p(x_{t}^\ell|x_{1:t-1}^\ell)$ to be  
    $$
        \delta_{j;t,\ell}  = 
    \begin{cases}
     - \epsilon h_t^\ell/||h_t^\ell||,  &\text{for~~} {j = x_{t}^\ell} \\ 
     0, & \text{for~~}{j \neq x_{t}^\ell}, 
    \end{cases}
    $$
    where $h_t^\ell$ is the RNN hidden state related to $x_{1:t-1}^\ell$, define in \eqref{equ:fht}.  
    \STATE Update $\{\vv\theta, \vv w\}$ using gradient ascent of log-likelihood~\eqref{equ:mle} on minibatch $\mathcal M$, 
    \ENDWHILE 
\STATE 
\emph{Remark}. We find it is practically useful to choose $\alpha$ to adapt with the norm of $w_i$, that is, 
$\epsilon = \alpha ||w_i||$ for each word, and $\alpha$ is a hyperparameter.     
\end{algorithmic}
\caption{Adversarial MLE Training} 
\label{alg:main} 
\end{algorithm}

\begin{thm}
\label{thm:diversity}
Given a set of embedding vectors $\vv w = \{w_i\}_{i\in \mathcal V}$, 
 if a word $w_i$ is $\epsilon$-recognizable, then we must have 
$$
\min_{j\neq i}||w_j- w_i||  >  \epsilon, 
$$
that is, 
$w_i$ is separated from the embedding vectors of all other words by at least $\epsilon$ distance. 
\end{thm}
\begin{proof}
If there exists $j\neq i$ such that $||w_j-w_i||\leq \epsilon$, 
following the adversarial optimization, we must have 
$$
w_j^\top h \geq  \min_{||\delta_i||\leq \epsilon} (w_i + \delta_i)^\top h > w_j^\top h.
$$
which forms a contradiction. 
\end{proof}
%
%
Note that maximizing the adversarial training objective function 
can be viewed as enforcing each $w_i$ to be $\epsilon$-recognized by 
its corresponding context vector $h$, and hence implicitly enforces diversity between the recognized words and the other words.   
We should remark that the context vector $h$ in Definition~\ref{def:word} does not have to exist in the training set, 
although it will more likely happen in the training set due to the training. 
 
In fact, we can draw a more explicit connection between pairwise distance and adversarial softmax function. 
\begin{thm} 
Following the definition in \eqref{equ:advsoft}, we have  
$$
\advsoft_\epsilon(i, \vv w, h) 
\leq \sigma \left ( \Phi(i, \vv w, ||h||) \right ), 
$$
where $\sigma(t) = 1/(1+\exp(-t))$ is the sigmoid function and 
$\Phi(i, \vv w, \alpha)$
is an ``energy function''
that measures the distance from $w_i$ to the other words $w_j$, $\forall j\neq i$: 
\begin{align*}
\Phi(i, \vv w, \alpha) 
& = -\log \sum_{j\neq i} \exp(-\alpha (||w_i - w_j|| - \epsilon)) \\
& \leq  \alpha \min_{j\neq i}(||w_i - w_j|| - \epsilon).  
\end{align*}
\label{thm:diversity}
\end{thm}
\vspace{-3em}
\begin{proof}
We have 
\begin{align*}
\begin{split}
& {\advsoft}_\epsilon(i, \vv w,h) \\
&~~~~~~~~~~~~~~   = \frac{\displaystyle\exp(w_i^\top h-\epsilon||h||)}{\displaystyle
\exp(w_i^\top h-\epsilon||h||) + \sum_{j\neq i} \exp(w_j^\top h)}\\
&~~~~~~~~~~~~~~   =
\sigma\left (\Psi(i, \vv w, h)\right),
\end{split}
\end{align*}
where 
$$
\Psi(i, \vv w, h)
= -\log \sum_{j\neq i}
\exp((w_j-w_i)^\top h + \epsilon ||h||).
$$
Note that $(w_j-w_i)^\top h \geq -||w_j-w_i|| \cdot ||h||$, we have
\begin{align*}
\Psi(i, \vv w, h)
& = -\log \sum_{j\neq i}
\exp((w_j-w_i)^\top h + \epsilon ||h||) \\
& \leq 
 -\log \sum_{j\neq i}
\exp(-||w_j-w_i||\cdot||h|| + \epsilon ||h||) \\
& = \Phi(i,\vv w, ||h||). 
\end{align*}
\vspace{-1.0em}
\end{proof}
\vspace{-0.5em}

Therefore, maximizing 
$\advsoft_\epsilon(i, \vv w, h)$, as our algorithm advocates, 
also maximizes the energy function $\Phi(i, \vv w, ||h||)$ to enforce $\min_{j\neq i}(||w_i - w_j||)$ larger than $\epsilon$ by placing a higher penalty 
on cases in which this is violated. 

\begin{figure*}
\centering
\begin{tabular}{ccc}
\raisebox{3em}{\rotatebox{90}{ Density}}
\includegraphics[height=0.2\textwidth]{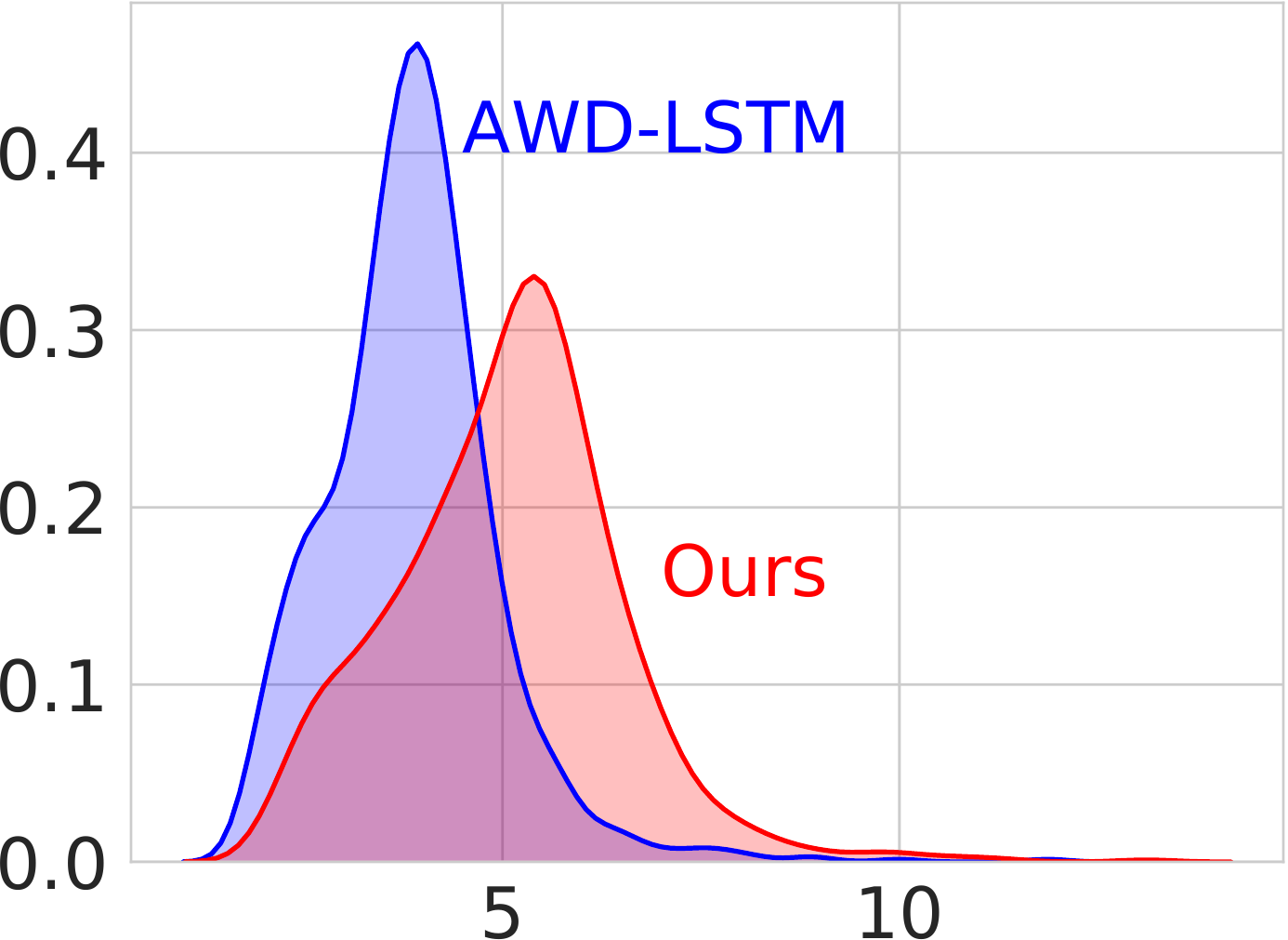} &
\raisebox{1.4em}{\rotatebox{90}{ Log Eigenvalues}}
\includegraphics[height=0.2\textwidth]{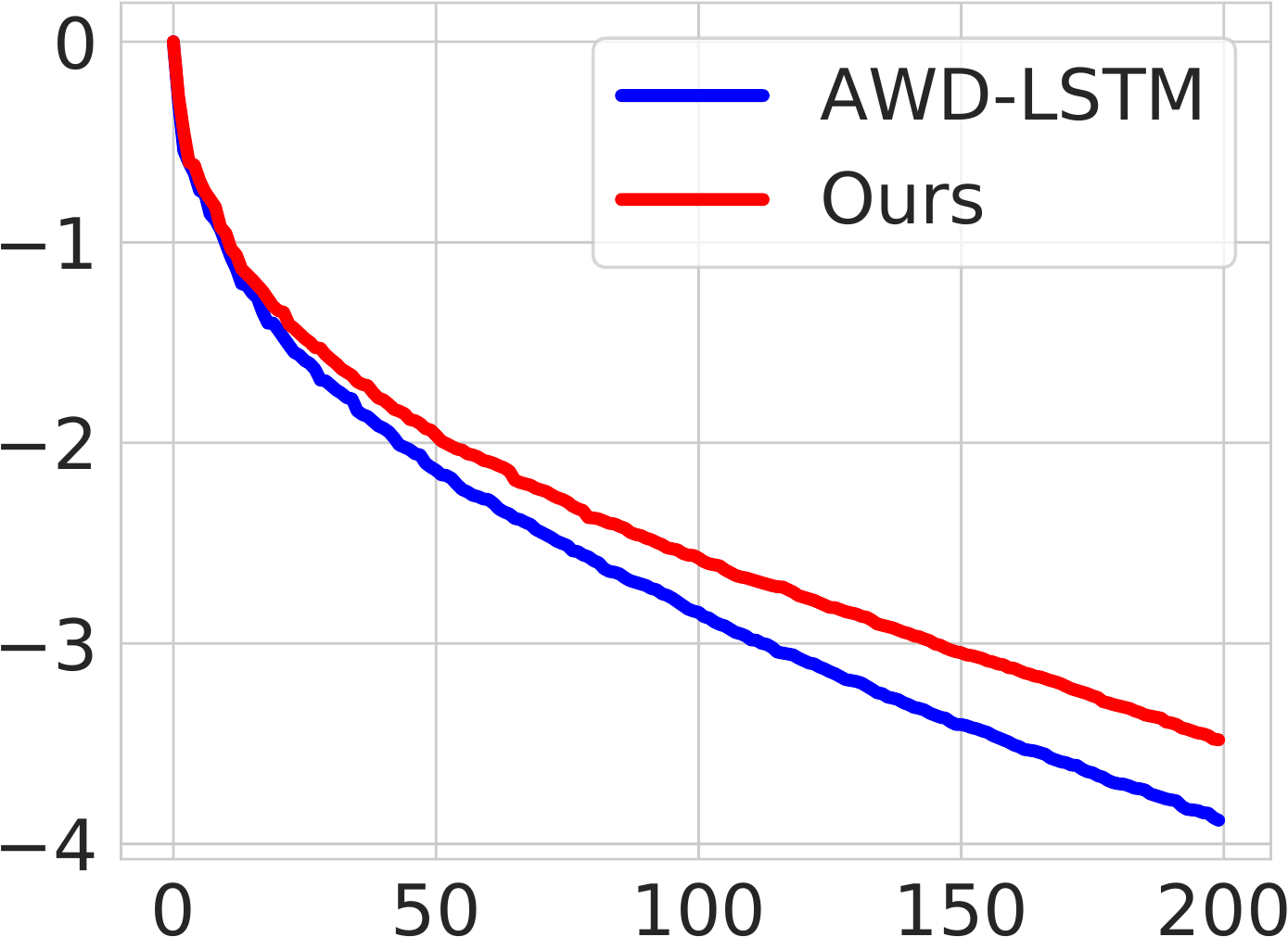} & 
\raisebox{3em}{\rotatebox{90}{ Perplexity}}
\includegraphics[height=0.2\textwidth]{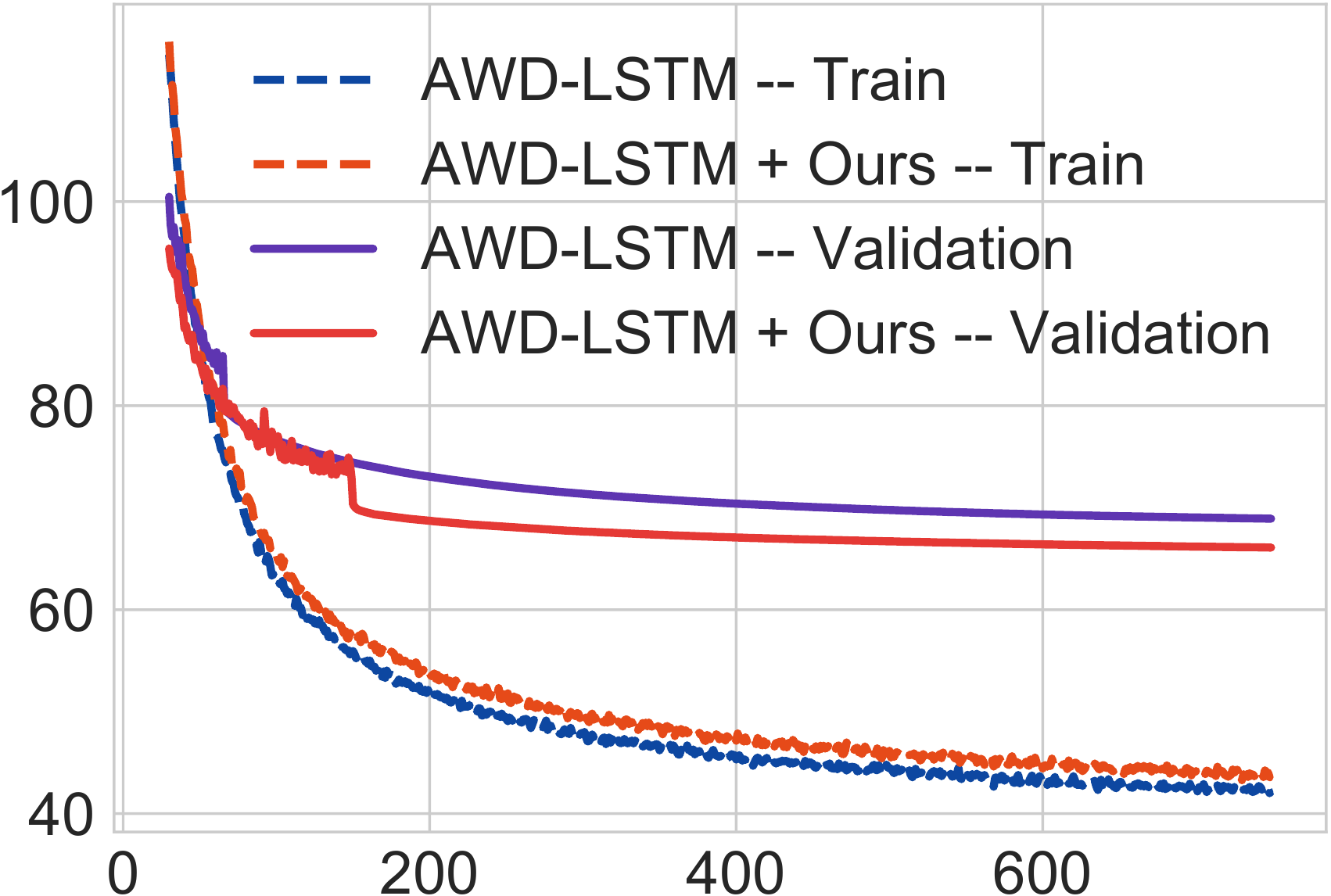}  
\\
(a) L2 Distance & (b) Index &  (c) Epochs (Wiki2)  \\
\end{tabular}
\caption{
(a) Kernel density estimation of the Euclidean distance to the nearest neighbor for each word;
(b) Logarithmic scale singular values of embedding
matrix. We normalize the singular values of each matrix so that the largest one is 1;
(c) Training and validation perplexities vs. training epochs
for AWD-LSTM \citep{merity2017regularizing} and our approach 
on the Wikitext-2(WT2) datasets. 
We follow the training settings reported in \citet{merity2017regularizing}.
The kink in the middle represents the start of fine-tuning.}
\label{fig:ptb_wt_ppl}
\end{figure*}

\section{Related Works and Discussions}
\label{sec:related}

\paragraph{Adversarial training} 
Adversarial machine learning has been an active 
research area recently \citep{szegedy2013intriguing, goodfellow2015explaining, athalye2018obfuscated},
in which algorithms are developed to 
either attack existing models by constructing adversarial examples, or 
 train robust models to defend adversarial attacks. 
More related to our work, 
\citep{sankaranarayanan2018regularizing} proposes a layer-wise adversarial training method to regularize deep neural networks. 
In statistics learning and robust statistics, 
various adversarial-like ideas are also leveraged to construct efficient and robust estimators, mostly for preventing model specification or data corruption \citep[e.g.,][]{maronna2018robust, duchi2016statistics}.     
Compared to these works, our work  
leverages the adversarial idea
as a regularization technique specifically for neural language models and focuses on introducing adversarial noise only on the softmax layers, so that a simple closed form solution can be obtained.
%

\paragraph{Direct Diversity Regularization}  
There has been a body of literature on
increasing the robustness by 
directly 
adding various forms of diversity-enforcing penalty functions \citep[e.g.,][]{elsayed2018large, xie2016diverse, liu2016large, liu2017sphereface, chen2017noisy, wang2018additive}. 
In the particular setting of enforcing diversity of word embeddings, 
\citet{gao2018representation} show 
that adding a cosine similarity regularizer improves language modeling performance, which has the form 
$
\sum_{i=1}^{|\vb|}\sum_{j\neq i}^{|\vb|} \frac{w_i^\top w_j }{ ||w_i||~||w_j||}.
$
However, in language modeling, 
one disadvantage of the direct diversity regularization approach  is that 
the vocabulary size $|\vb|$ can be huge, and
calculating the summation term exactly at each step is not feasible, while 
approximation with mini-batch samples may  make it ineffective. 
Our method promotes diversity implicitly with theoretical guarantees and does not introduce  computational overhead. 

\paragraph{Large-margin classification}
%
 In a general sense, our method can be seen as an instance of
 constructing
 large-margin classifiers by enforcing the distance of a word to its neighbors larger than a margin if it's recognized by any context.
 Learning large-margin classifiers has been extensively studied in the literature; 
 see e.g., \citet{weston1999support, tsochantaridis2005large, jiang2018predicting, elsayed2018large, liu2016large, liu2017sphereface}. 

\paragraph{Other Regularization Techniques for  
Language Models} 
Various other techniques have been also developed to address overfitting in RNN language models. 
For example, 
\citet{gal2016theoretically} propose to use variational inference-based dropout 
\citep{srivastava2014dropout} on recurrent neural networks, 
 in which the same dropout mask
 is repeated  at each time step for inputs, outputs, and recurrent layers  
for regularizing RNN models.  
\citet{merity2017regularizing} suggest to use DropConnect \citep{wan2013regularization} 
on the recurrent weight matrices and report a series of encouraging benchmark results.
Other types of regularization include 
activation regularization \citep{merity2017revisiting},
layer normalization \citep{ba2016layer}, 
and frequency agnostic training \citep{gong2018frage}, etc.
Our work is orthogonal to these regularization and optimization techniques
and can be easily combined with them to achieve further improvements, as we demonstrate in our experiments. 



\section{Empirical Results} 
\label{sec:experiement}
We demonstrate the effectiveness of our method
in two applications: neural language modeling and 
neural machine translation, and compare them with state-of-the-art architectures and learning methods.
All models are trained with the weight-tying trick \citep{press2016using, inan2016tying}. 
Our code is available at:~ \url{https://github.com/ChengyueGongR/advsoft}.

\begin{table*}[ht]
    \begin{center}
    \setlength{\tabcolsep}{1.2em}
        \begin{tabular}{l|c||cc}
            \toprule
            Method & \bf Params & Valid & Test \\
            \hline
            Variational LSTM~\cite{gal2016theoretically} & 19M & - & 73.4 \\
            Variational LSTM + weight tying~\cite{inan2016tying} & 51M & 71.1 & 68.5 \\
            NAS-RNN~\cite{zoph2016neural} & 54M & - & 62.4 \\
            DARTS~\cite{liu2018darts} & 23M & 58.3 & 56.1 \\
            \hline
            \multicolumn{4}{l}{w/o dynamic evaluation} \\ 
            \hline
            AWD-LSTM ~\citep{merity2017regularizing} & 24M & 60.00 & 57.30\\
            \bf{AWD-LSTM + Ours} & 24M & \bf{57.15} & \bf{55.01}\\
             AWD-LSTM + MoS \citep{yang2017breaking} & 22M & 56.54 & 54.44 \\
            \bf AWD-LSTM + MoS + Ours & 22M & \bf{54.98} & \bf{52.87} \\
            AWD-LSTM + MoS + Partial Shuffled \citep{press2019partially} & 22M & 55.89 & 53.92 \\
            \bf AWD-LSTM + MoS + Partial Shuffled + Ours & 22M & \bf{54.10} & \bf{52.20} \\
            \hline
            \multicolumn{4}{l}{+~~dynamic evaluation \citep{krause2017dynamic}} \\ 
            \hline
            AWD-LSTM 
            ~\citep{merity2017regularizing}   & 24M & 51.60 & 51.10\\
            \bf{AWD-LSTM +Ours}  & 24M & \bf{49.31} & \bf{48.72}\\
             AWD-LSTM + MoS 
            \citep{yang2017breaking}  & 22M & 48.33 & 47.69 \\
            \bf AWD-LSTM + MoS + Ours & 22M & \bf{47.15} & \bf{46.52} \\
            AWD-LSTM + MoS + Partial Shuffled \citep{press2019partially} & 22M & 47.93 & 47.49 \\
            \bf AWD-LSTM + MoS + Partial Shuffled + Ours & 22M & \bf{46.63} & \bf{46.01} \\
            \bottomrule
        \end{tabular}
    \end{center}
\caption{\label{PTB-table}
Perplexities on the validation and test sets on the Penn Treebank dataset.  Smaller perplexities refer to better language modeling performance. {\tt Params} denotes the number of model parameters. 
}
\end{table*}

\begin{table*}[ht]
    \begin{center}
    \setlength{\tabcolsep}{1.2em}
        \begin{tabular}{l|c||cc}
            \toprule
            Method & \bf Params & Valid & Test\\
            \hline
            Variational LSTM~\cite{inan2016tying} (h = 650)  &28M &92.3& 87.7\\
			Variational LSTM~\cite{inan2016tying} (h = 650) + weight tying &28M& 91.5& 87.0\\
			1-layer LSTM~\cite{mandt2017stochastic} &24M& 69.3 &65.9\\
			2-layer skip connection LSTM ~\cite{mandt2017stochastic} (tied) &24M &69.1& 65.9\\
			DARTS~\cite{liu2018darts} & 33M & 69.5 & 66.9 \\
			\hline
            \hline
            \multicolumn{4}{l}{w/o dynamic evaluation} \\
            \hline
            AWD-LSTM ~\citep{merity2017regularizing} & 33M & 68.60 & 65.80\\
           \bf{AWD-LSTM + Ours} & 33M & \bf{64.01} & \bf{61.56}\\
            AWD-LSTM + MoS \citep{yang2017breaking} & 35M & 63.88 & 61.45 \\
            \bf AWD-LSTM + MoS + Ours & 35M & \bf{61.93} & \bf{59.62} \\
            AWD-LSTM + MoS + Partial Shuffled \citep{press2019partially} & 35M & 62.38 & 59.98 \\
            \bf AWD-LSTM + MoS + Partial Shuffled + Ours & 35M & \bf{61.10} & \bf{58.95} \\
            \hline
            \multicolumn{4}{l}{+~~dynamic evaluation \citep{krause2017dynamic}} \\ 
            \hline
             AWD-LSTM ~\citep{merity2017regularizing}  & 33M & 46.40 & 44.30\\
            \bf{AWD-LSTM + Ours}  & 33M & \bf{42.48} & \bf{40.71}\\
            AWD-LSTM + MoS \citep{yang2017breaking}  & 35M & 42.41 & 40.68 \\
            \bf AWD-LSTM + MoS + Ours  & 35M & \bf{40.27} & \bf{38.65} \\
            AWD-LSTM + MoS + Partial Shuffled \citep{press2019partially} & 35M & 40.75 & 39.03 \\
            \bf AWD-LSTM + MoS + Partial Shuffled + Ours & 35M & \bf 39.58 & \bf 38.07 \\
            \bottomrule
        \end{tabular}
    \end{center}
\caption{\label{WT2-table} Perplexities on validation and test sets on the Wikitext-2 dataset.} 
\end{table*}

\begin{table*}[ht]
    \begin{center}
    \setlength{\tabcolsep}{1.5em}
        \begin{tabular}{l|cc}
            \toprule
            Method & Valid & Test\\
            \hline
            LSTM~\cite{grave2016efficient} & - & 48.7 \\
            Temporal CNN~\cite{bai2018convolutional} & - & 45.2 \\
            GCNN~\cite{dauphin2016language}  & - & 37.2\\
            LSTM + Hebbian~\citep{rae2018fast} & 34.1 &  34.3 \\
            \hline
             4 layer QRNN~\cite{merity2018analysis} & 32.0 & 33.0\\
             \bf{4 layer QRNN + Ours}  & \bf{30.6} & \bf{31.6}\\
             \hline \multicolumn{3}{l}{+~~post process \cite{rae2018fast}} \\
             \hline
             LSTM + Hebbian + Cache + MbPA~\cite{rae2018fast} & 29.0 & 29.2\\
             \bf{4 layer QRNN + Ours + dynamic evaluation}  & \bf{27.2} & \bf{28.0}\\
            \bottomrule
        \end{tabular}
    \end{center}
\caption{\label{WT103-table} Perplexities on validation and test sets on the  Wikitext-103 dataset. 
}
\end{table*}

\subsection{Experiments on Language Modeling}
We test our method on three benchmark datasets: 
Penn Treebank (PTB), Wikitext-2 (WT2) and Wikitext-103 (WT103). 


\paragraph{PTB
}
The PTB corpus~\citep{marcus1993building} 
has been a standard dataset used for benchmarking language models. It consists of 923k training, 73k validation and 82k test words.
We use the processed version provided by \citet{mikolov2010recurrent}
that is widely used for this dataset \citep[e.g.,][]{merity2017regularizing, yang2017breaking, kanai2018sigsoftmax, gong2018sentence}.

\paragraph{WT2 and WT103
} 
The WT2 and WT103 datasets are introduced in \citet{merity2016pointer} as an alternative to the PTB dataset, 
and which contain lightly pre-possessed Wikipedia articles. 
The WT2 and WT103 contain approximately 2 million and 103 million words, respectively. 

\paragraph{Experimental settings}
For the PTB and WT2 datasets, 
we closely follow the regularization and optimization techniques introduced in
AWD-LSTM~\citep{merity2017regularizing},
which stacks a three-layer LSTM and 
performs optimization with a bag of tricks.
%



The WT103 corpus contains around 103 million tokens, which is significantly larger than 
the PTB and WT2 datasets. 
In this case, we use Quasi-Recurrent neural networks (QRNN)-based language models \citep{merity2018analysis, bradbury2016quasi} as our base model for efficiency. 
QRNN allows for parallel computation across both time-step and minibatch dimensions,
enabling high throughput and good scaling for long sequences and large datasets. 

\citet{yang2017breaking} show that softmax-based language models yield low-rank
approximations and do not have enough capacity to model complex natural language.
They propose a mixture of softmax (MoS) to break the softmax bottleneck and achieve
significant improvements. 
We also evaluated our method within the MoS framework by 
directly following the experimental settings in \citet{yang2017breaking}, 
except we replace the original softmax function with our adversarial softmax function.

The training procedure of AWD-LSTM-based language models can be decoupled into two stages: 1) optimizing the model with SGD and averaged SGD (ASGD); 
2) restarting ASGD for fine-tuning.
We report the perplexity scores at the end of both stages.
We also report the perplexity scores with a recent proposed post-process method, dynamical evaluation \citep{krause2017dynamic}
after fine-tuning.


\paragraph{Applying Adversarial MLE training} 
To investigate the effectiveness of our approach, we simply replace the softmax layer of baseline methods with our adversarial softmax function, 
with all other the parameters and architectures untouched. 
We empirically found that adding small annealed Gaussian noise in the input embedding layer makes our noisy model converge more quickly.
We experimented with different ways of scaling the Gaussian dropout level and
found that a small Gaussian noise with zero mean and a small variance,
such that it decreases from 0.2 to 0.0 over the duration of the run, 
works well for all the tasks.

Note the optimal adversarial noise $\delta_i = -\epsilon h/||h||$ (see Algorithm~\ref{alg:main}) given RNN
prediction $h$ associated with a target word $w_i$.
Here,  $\epsilon$ controls the magnitude of of the noise level. 
When $\epsilon = 0$, our approach reduces to the original MLE training.  
We propose setting the noise level adaptive 
that proportional to the L2 norm of the target word embeddings, namely, by setting $\epsilon = \alpha ||w_i||$ with $\alpha$ as a hyperparameter. 

Figure~\ref{fig:adversarial_noise} shows the training and validation perplexities 
on the PTB dataset with different choices of $\alpha$. 
 We find that $\alpha$ in the range of $[0.001,0.01]$ perform similarly well.   
 Larger values (e.g., $\alpha=0.05$) causes more difficult optimization and hence underfitting, 
 while smaller values (e.g., $\alpha=0$ (the baseline approach))
 tends to overfit as we observe from standard MLE training. 
 We set $\alpha= 0.005$ for the rest of experiments  unless otherwise specified. 

\begin{figure}[ht]
    \centering
     \setlength\tabcolsep{2pt} 
    \begin{tabular}{cc}
    \raisebox{1em}{\rotatebox{90}{\small Training Perplexity}}
    \includegraphics[width=0.21\textwidth]{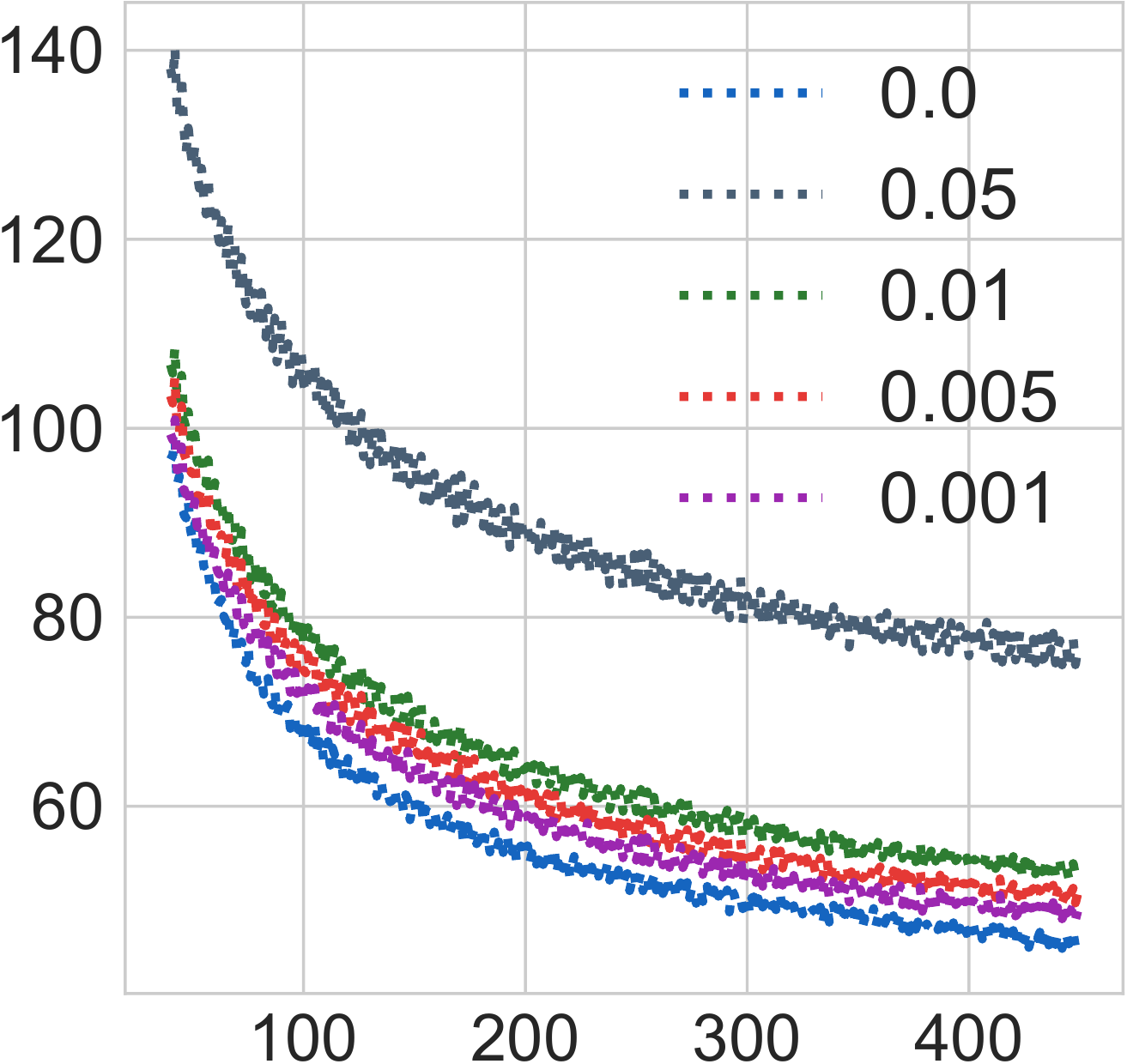} &
    \raisebox{1em}{\rotatebox{90}{\small Validation Perplexity}}
    \includegraphics[width=0.21\textwidth]{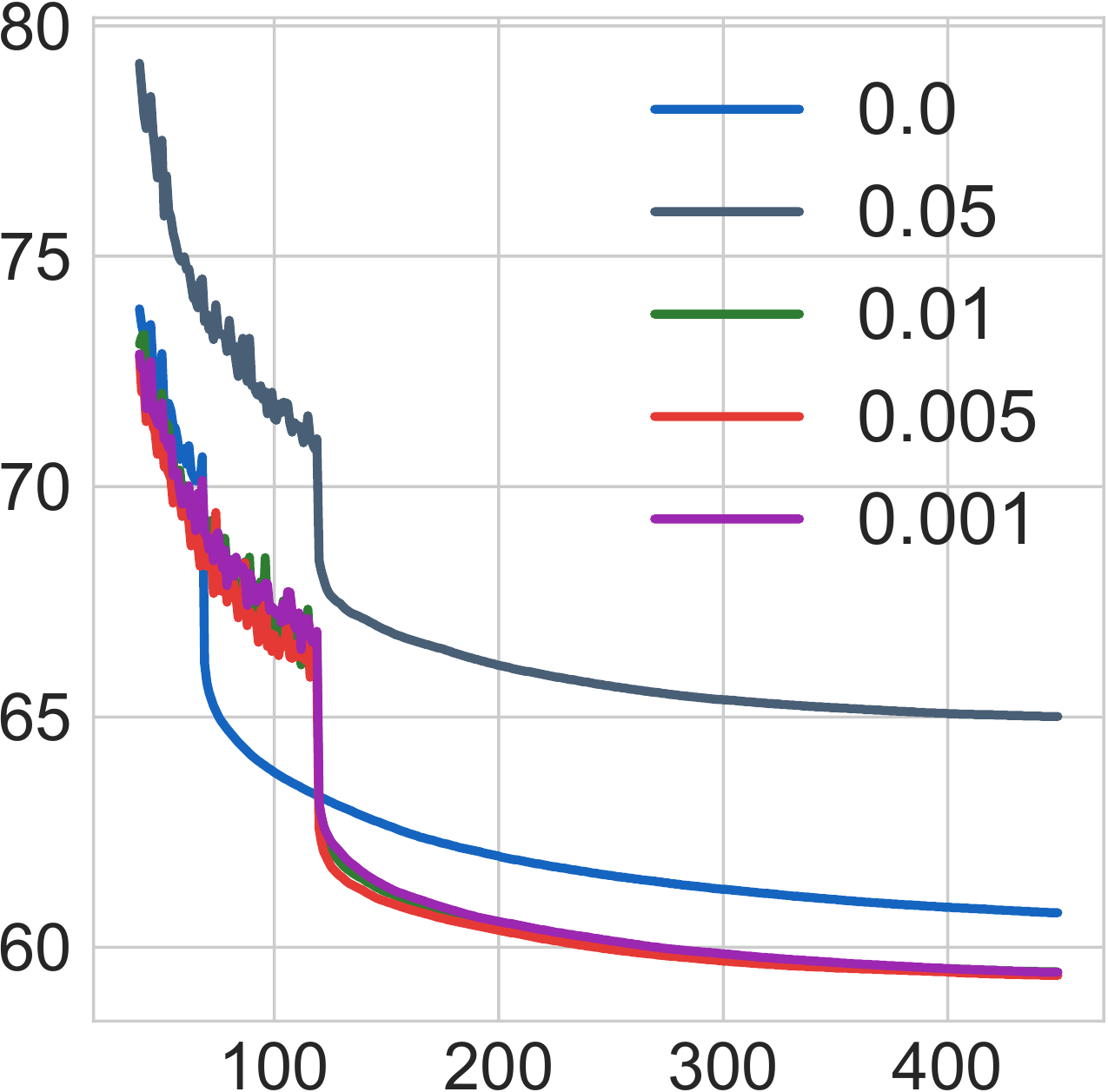}\\
     Epochs & Epochs \\
    \end{tabular}
    \caption{Training and validation perplexities on the PTB dataset with different choices of adversarial perturbation magnitude.}
    \label{fig:adversarial_noise}
\end{figure}


\paragraph{Results on PTB and WT2}
The results on the PTB and WT2 corpus are illustrated in Tables~\ref{PTB-table}
and \ref{WT2-table}, respectively. 
Methods with our adversarial softmax outperform
the baselines in all settings.  
Our results establish a new single model state-of-the-art
on PTB and WT2, achieving perplexity scores of 46.01 and 38.07, respectively.
Specifically, our approach significantly improves AWD-LSTM by a margin of  
2.29/2.38 and 3.92/3.59 in validation and test perplexity on the PTB and WT2 dataset.
We also improve the AWD-LSTM-MoS baseline 
by an amount of  1.18/1.17 and 2.14/2.03 in perplexity
for both datasets.


\paragraph{Results on WT103}
Table~\ref{WT103-table} shows that
on the large-scale WT103 dataset,
we improve the QRNN baseline with 1.4/1.4 points in perplexity on validation and test sets, respectively. 
With dynamic evaluation,
our method can achieve a test perplexity of 28.0,
which is,
to the authors' knowledge,
better than all existing CNN- or RNN-based models 
with similar numbers of model parameters.



\paragraph{Analysis}
We further analyze the properties of the learned word embeddings on the WT2 dataset.
Figure~\ref{fig:ptb_wt_ppl} (a) shows the distribution (via kernel density estimation) of the L2 distance between each word and its nearest neighbor learned by our method and the baseline, 
which verifies the diversity promoting property of our method. 
Figure~\ref{fig:ptb_wt_ppl} (b) shows the  
singular values of word embedding matrix learned by our model and that by the baseline model. 
We can see that, when trained with our method, the singular values distribute more uniformly,
an indication that our embedding vectors fills a higher dimensional subspace. 

 Figure~\ref{fig:ptb_wt_ppl} (c) 
 shows the training and validation 
 perplexities  of our method and baseline on AWD-LSTM. 
 We can see that our method is less prone to overfitting. While the baseline model reaches a smaller training error quickly, our method has a larger training error at the same stage because it optimizes a more difficult adversarial  objective, yet yields a significantly lower 
 validation error. 



\subsection{Experiments on Machine Translation}
We apply our method on machine translation tasks. Neural machine translation aims at building a single neural network
that maximize translation performance. 
Given a source sentence $s$, translation is equivalent to finding a target 
sentence $t$ by maximizing the conditional probability $p(t|s)$.
Here, we fit a parametrized model to maximize the conditional probability
using a parallel training corpus. Specifically, we use an RNN encoder-decoder framework \citep{cho2014learning, gehring2017convolutional, vaswani2017attention}, upon which we apply 
our adversarial MLE training that learns to translate. 


\paragraph{Datasets} 
We evaluate the proposed method on two translation tasks:
WMT2014 English $\to$ German (En$\to$De) and IWSLT2014 German $\to$ English (De$\to$En) translation. 
We use the parallel corpora publicly available at WMT 2014  and IWSLT 2014, 
which have been widely used for benchmark neural machine translation tasks~\cite{vaswani2017attention,gehring2017convolutional}. 
For fair comparison, we follow the standard data pre-processing procedures described in \citet{ranzato2015sequence, bahdanau2016actor}.

\noindent\emph{WMT2014 En$\to$De}
We use the original training set for model training, which consists of 
4.5 million sentence pairs. Source and target sentences are encoded by 
37K shared sub-word tokens based on byte-pair encoding (BPE) \citep{sennrich2015neural}.
We use the concatenation of newstest2012 and newstest2013 
as the validation set and test on newstest2014. 

\noindent\emph{IWSLT2014 De$\to$En}
This dataset contains 160K training sequences pairs and 7K validation sentence
pairs. Sentences are encoded using BPE with a shared vocabulary of about 33K tokens.
We use the concatenation of dev2010, tst2010, tst2011 and tst2011 as the test set, which is widely used in prior works~\citep{bahdanau2016actor}.

\paragraph{Experimental settings}
We choose the Transformer-based state-of-the-art machine translation model 
~\citep{vaswani2017attention} 
as our base model and use \textit{Tensor2Tensor}~\citep{tensor2tensor}~\footnote{https://github.com/tensorflow/tensor2tensor}
for implementation.
Specifically, 
to be consistent with prior works,
we closely follow the settings reported in \citet{vaswani2017attention}.
We use the Adam optimizer~\cite{kingma2014adam}
and follow the learning rate warm-up strategy in~\citet{vaswani2017attention}.
Sentences are pre-processed using byte-pair encoding~\citep{BPE} into subword tokens before training, 
and we measure the final performance with the BLEU score.

For the WMT2014 De$\to$En task,
we evaluate on the \textit{Transformer-Base} and \textit{Transformer-Big} architectures,
which consist of a 6-layer encoder and a 6-layer decoder with 512-dimensional and 1024-dimensional hidden units per layer, respectively. 
For the IWSLT2014 De$\to$En task, 
we evaluate on two standard configurations: \textit{Transformer-Small} and \textit{Transformer-Base}.
For \textit{Transformer-Small}, we stack a 4-layer encoder and a 4-layer decoder with 256-dimensional hidden units per layer. For \textit{Transformer-Base}, we set the batch size to 6400 and the dropout rate to 0.4 following \citet{wang2018multiagent}. For both tasks, we share the BPE subword vocabulary for decoder and encoder.

\paragraph{Results}
From Table~\ref{WMT-table} and Table~\ref{IWSLT-table}, 
we can see that our method
improves over the baseline algorithms for all settings.  
On the WMT2014 De$\to$En translation task, 
our method reaches 28.43 and 29.52 in BLEU score 
with the \textit{Transformer Base} and \textit{Transformer Big} architectures,  respectively; 
this yields an 1.13/1.12 improvement over their corresponding baseline models. 
On the IWSLT2014 De$\to$En dataset,
our method improves the BLEU score from 32.47 to  33.61 and 34.43 to 35.18 for the  \textit{Transformer-Small} and \textit{Transformer-Base} configurations, respectively.

\begin{table}[ht]
  	\begin{center}
  	\setlength\extrarowheight{1.5pt}
	\begin{tabular}{lc}
\toprule 
\bf  Method & \bf  BLEU \\
\hline
Local Attention~\cite{luong2015effective} &  20.90 \\
ByteNet~\cite{kalchbrenner2016neural} &  23.75 \\
ConvS2S~\cite{gehring2017convolutional} &  25.16 \\
\hline
Transformer Base ~\cite{vaswani2017attention}&  27.30 \\
\bf{Transformer Base + Ours} &  \bf {28.43} \\
\hline
Transformer Big ~\cite{vaswani2017attention}&  28.40 \\
Transformer Big + ~\cite{gao2018representation} &  28.94 \\
\bf{Transformer Big + Ours} &  \bf{29.52} \\
\bottomrule 
\end{tabular}
\end{center}
\caption{\label{WMT-table}  BLEU scores on the WMT2014 Ee$\to$De machine translation task. }
\end{table}

\begin{table}[ht]
  	\begin{center}
  	\setlength\extrarowheight{1.5pt}
	\begin{tabular}{lc}
\toprule 
\bf  Method & \bf  BLEU \\
\hline
Actor-critic~\cite{bahdanau2016actor}  &  28.53 \\
CNN-a~\cite{gehring2016convolutional}  &  30.04\\
\hline
Transformer Small~\cite{vaswani2017attention} &  32.47 \\
\bf{Transformer Small + Ours} &  \bf 33.61 \\
\hline
Transformer Base + \cite{wang2018multiagent} &  34.43 \\
\bf{Transformer Base + Ours} &  \bf 35.18 \\
\bottomrule
\end{tabular}
\end{center}
\caption{\label{IWSLT-table} BLEU scores on the IWSLT2014 De$\to$En machine translation task. }
\end{table}

\section{Conclusion}
In this work, we present an adversarial MLE training strategy for  neural language modeling,  which promotes diversity in the embedding space and improves the generalization performance.
Our approach can be easily used as a drop-in  replacement for standard MLE-based model
with no additional training parameters and  computational overhead. 
Applying this approach to a variety of language modeling and machine translation tasks, we achieve improvements over state-of-the-art baseline models on standard benchmarks. 

\section*{Acknowledgement}
This work is supported in part by NSF CRII 1830161 and NSF  CAREER 1846421.
We would like to acknowledge Google Cloud for their support.

\bibliographystyle{iclr2019_conference}
\bibliography{advsoft}

\begin{thebibliography}{66}
\providecommand{\natexlab}[1]{#1}
\providecommand{\url}[1]{\texttt{#1}}
\expandafter\ifx\csname urlstyle\endcsname\relax
  \providecommand{\doi}[1]{doi: #1}\else
  \providecommand{\doi}{doi: \begingroup \urlstyle{rm}\Url}\fi

\bibitem[Anderson et~al.(2018)Anderson, He, Buehler, Teney, Johnson, Gould, and
  Zhang]{anderson2018bottom}
Peter Anderson, Xiaodong He, Chris Buehler, Damien Teney, Mark Johnson, Stephen
  Gould, and Lei Zhang.
\newblock Bottom-up and top-down attention for image captioning and visual
  question answering.
\newblock In \emph{CVPR}, volume~3, pp.\ ~6, 2018.

\bibitem[Athalye et~al.(2018)Athalye, Carlini, and
  Wagner]{athalye2018obfuscated}
Anish Athalye, Nicholas Carlini, and David Wagner.
\newblock Obfuscated gradients give a false sense of security: Circumventing
  defenses to adversarial examples.
\newblock \emph{Proceedings of the 35th International Conference on Machine
  Learning, {ICML}}, 2018.

\bibitem[Ba et~al.(2016)Ba, Kiros, and Hinton]{ba2016layer}
Jimmy~Lei Ba, Jamie~Ryan Kiros, and Geoffrey~E Hinton.
\newblock Layer normalization.
\newblock \emph{arXiv preprint arXiv:1607.06450}, 2016.

\bibitem[Bahdanau et~al.(2017)Bahdanau, Brakel, Xu, Goyal, Lowe, Pineau,
  Courville, and Bengio]{bahdanau2016actor}
Dzmitry Bahdanau, Philemon Brakel, Kelvin Xu, Anirudh Goyal, Ryan Lowe, Joelle
  Pineau, Aaron Courville, and Yoshua Bengio.
\newblock An actor-critic algorithm for sequence prediction.
\newblock \emph{International Conference on Learning Representations, ICLR},
  2017.

\bibitem[Bai et~al.(2018)Bai, Kolter, and Koltun]{bai2018convolutional}
Shaojie Bai, J.~Zico Kolter, and Vladlen Koltun.
\newblock Convolutional sequence modeling revisited, 2018.
\newblock URL \url{https://openreview.net/forum?id=rk8wKk-R-}.

\bibitem[Bradbury et~al.(2017)Bradbury, Merity, Xiong, and
  Socher]{bradbury2016quasi}
James Bradbury, Stephen Merity, Caiming Xiong, and Richard Socher.
\newblock Quasi-recurrent neural networks.
\newblock \emph{International Conference on Learning Representations, ICLR},
  2017.

\bibitem[Chen et~al.(2017)Chen, Deng, and Du]{chen2017noisy}
Binghui Chen, Weihong Deng, and Junping Du.
\newblock Noisy softmax: Improving the generalization ability of dcnn via
  postponing the early softmax saturation.
\newblock In \emph{The IEEE Conference on Computer Vision and Pattern
  Recognition, CVPR}, 2017.

\bibitem[Cho et~al.(2014)Cho, Van~Merri{\"e}nboer, Gulcehre, Bahdanau,
  Bougares, Schwenk, and Bengio]{cho2014learning}
Kyunghyun Cho, Bart Van~Merri{\"e}nboer, Caglar Gulcehre, Dzmitry Bahdanau,
  Fethi Bougares, Holger Schwenk, and Yoshua Bengio.
\newblock Learning phrase representations using rnn encoder-decoder for
  statistical machine translation.
\newblock \emph{Conference on Empirical Methods in Natural Language Processing,
  EMNLP,}, 2014.

\bibitem[Chung et~al.(2014)Chung, Gulcehre, Cho, and
  Bengio]{chung2014empirical}
Junyoung Chung, Caglar Gulcehre, KyungHyun Cho, and Yoshua Bengio.
\newblock Empirical evaluation of gated recurrent neural networks on sequence
  modeling.
\newblock \emph{arXiv preprint arXiv:1412.3555}, 2014.

\bibitem[Cogswell et~al.(2016)Cogswell, Ahmed, Girshick, Zitnick, and
  Batra]{cogswell2015reducing}
Michael Cogswell, Faruk Ahmed, Ross Girshick, Larry Zitnick, and Dhruv Batra.
\newblock Reducing overfitting in deep networks by decorrelating
  representations.
\newblock \emph{ICLR}, 2016.

\bibitem[Cortes \& Vapnik(1995)Cortes and Vapnik]{cortes1995support}
Corinna Cortes and Vladimir Vapnik.
\newblock Support-vector networks.
\newblock \emph{Machine learning}, 20\penalty0 (3):\penalty0 273--297, 1995.

\bibitem[Dauphin et~al.(2016)Dauphin, Fan, Auli, and
  Grangier]{dauphin2016language}
Yann~N Dauphin, Angela Fan, Michael Auli, and David Grangier.
\newblock Language modeling with gated convolutional networks.
\newblock \emph{arXiv preprint arXiv:1612.08083}, 2016.

\bibitem[Duchi et~al.(2016)Duchi, Glynn, and Namkoong]{duchi2016statistics}
John Duchi, Peter Glynn, and Hongseok Namkoong.
\newblock Statistics of robust optimization: A generalized empirical likelihood
  approach.
\newblock \emph{arXiv preprint arXiv:1610.03425}, 2016.

\bibitem[Elsayed et~al.(2018)Elsayed, Krishnan, Mobahi, Regan, and
  Bengio]{elsayed2018large}
Gamaleldin~F Elsayed, Dilip Krishnan, Hossein Mobahi, Kevin Regan, and Samy
  Bengio.
\newblock Large margin deep networks for classification.
\newblock \emph{in Advances in Neural Information Processing Systems}, 2018.

\bibitem[Gal \& Ghahramani(2016)Gal and Ghahramani]{gal2016theoretically}
Yarin Gal and Zoubin Ghahramani.
\newblock A theoretically grounded application of dropout in recurrent neural
  networks.
\newblock In \emph{Advances in neural information processing systems}, pp.\
  1019--1027, 2016.

\bibitem[Gao et~al.(2019)Gao, He, Tan, Qin, Wang, and
  Liu]{gao2018representation}
Jun Gao, Di~He, Xu~Tan, Tao Qin, Liwei Wang, and Tieyan Liu.
\newblock Representation degeneration problem in training natural language
  generation models.
\newblock \emph{ICLR}, 2019.

\bibitem[Gehring et~al.(2017{\natexlab{a}})Gehring, Auli, Grangier, and
  Dauphin]{gehring2016convolutional}
Jonas Gehring, Michael Auli, David Grangier, and Yann~N Dauphin.
\newblock A convolutional encoder model for neural machine translation.
\newblock \emph{ACL}, 2017{\natexlab{a}}.

\bibitem[Gehring et~al.(2017{\natexlab{b}})Gehring, Auli, Grangier, Yarats, and
  Dauphin]{gehring2017convolutional}
Jonas Gehring, Michael Auli, David Grangier, Denis Yarats, and Yann~N Dauphin.
\newblock Convolutional sequence to sequence learning.
\newblock \emph{ICML}, 2017{\natexlab{b}}.

\bibitem[Gong et~al.(2018)Gong, He, Tan, Qin, Wang, and Liu]{gong2018frage}
Chengyue Gong, Di~He, Xu~Tan, Tao Qin, Liwei Wang, and Tie-Yan Liu.
\newblock Frage: frequency-agnostic word representation.
\newblock In \emph{Advances in Neural Information Processing Systems}, pp.\
  1339--1350, 2018.

\bibitem[Gong et~al.(2019)Gong, Tan, He, and Qin]{gong2018sentence}
Chengyue Gong, Xu~Tan, Di~He, and Tao Qin.
\newblock Sentence-wise smooth regularization for sequence to sequence
  learning.
\newblock \emph{AAAI}, 2019.

\bibitem[Goodfellow et~al.(2015)Goodfellow, Shlens, and
  Szegedy]{goodfellow2015explaining}
Ian Goodfellow, Jonathon Shlens, and Christian Szegedy.
\newblock Explaining and harnessing adversarial examples.
\newblock In \emph{International Conference on Learning Representations}, 2015.

\bibitem[Grave et~al.(2017)Grave, Joulin, Ciss{\'e}, Grangier, and
  J{\'e}gou]{grave2016efficient}
Edouard Grave, Armand Joulin, Moustapha Ciss{\'e}, David Grangier, and
  Herv{\'e} J{\'e}gou.
\newblock Efficient softmax approximation for gpus.
\newblock \emph{International Conference on Machine Learning, ICML}, 2017.

\bibitem[Hochreiter \& Schmidhuber(1997)Hochreiter and
  Schmidhuber]{hochreiter1997long}
Sepp Hochreiter and J{\"u}rgen Schmidhuber.
\newblock Long short-term memory.
\newblock \emph{Neural computation}, 9\penalty0 (8):\penalty0 1735--1780, 1997.

\bibitem[Inan et~al.(2017)Inan, Khosravi, and Socher]{inan2016tying}
Hakan Inan, Khashayar Khosravi, and Richard Socher.
\newblock Tying word vectors and word classifiers: A loss framework for
  language modeling.
\newblock \emph{ICLR}, 2017.

\bibitem[Jiang et~al.(2018)Jiang, Krishnan, Mobahi, and
  Bengio]{jiang2018predicting}
Yiding Jiang, Dilip Krishnan, Hossein Mobahi, and Samy Bengio.
\newblock Predicting the generalization gap in deep networks with margin
  distributions.
\newblock \emph{arXiv preprint arXiv:1810.00113}, 2018.

\bibitem[Kalchbrenner et~al.(2016)Kalchbrenner, Espeholt, Simonyan, Oord,
  Graves, and Kavukcuoglu]{kalchbrenner2016neural}
Nal Kalchbrenner, Lasse Espeholt, Karen Simonyan, Aaron van~den Oord, Alex
  Graves, and Koray Kavukcuoglu.
\newblock Neural machine translation in linear time.
\newblock \emph{arXiv preprint arXiv:1610.10099}, 2016.

\bibitem[Kanai et~al.(2018)Kanai, Fujiwara, Yamanaka, and
  Adachi]{kanai2018sigsoftmax}
Sekitoshi Kanai, Yasuhiro Fujiwara, Yuki Yamanaka, and Shuichi Adachi.
\newblock Sigsoftmax: Reanalysis of the softmax bottleneck.
\newblock \emph{In Advances in Neural Information Processing Systems}, 2018.

\bibitem[Khodak et~al.(2018)Khodak, Saunshi, Liang, Ma, Stewart, and
  Arora]{khodak2018carte}
Mikhail Khodak, Nikunj Saunshi, Yingyu Liang, Tengyu Ma, Brandon Stewart, and
  Sanjeev Arora.
\newblock A la carte embedding: Cheap but effective induction of semantic
  feature vectors.
\newblock \emph{ACL}, 2018.

\bibitem[Kingma \& Ba(2014)Kingma and Ba]{kingma2014adam}
Diederik~P Kingma and Jimmy Ba.
\newblock Adam: A method for stochastic optimization.
\newblock \emph{arXiv preprint arXiv:1412.6980}, 2014.

\bibitem[Koehn(2009)]{koehn2009statistical}
Philipp Koehn.
\newblock \emph{Statistical machine translation}.
\newblock Cambridge University Press, 2009.

\bibitem[Krause et~al.(2018)Krause, Kahembwe, Murray, and
  Renals]{krause2017dynamic}
Ben Krause, Emmanuel Kahembwe, Iain Murray, and Steve Renals.
\newblock Dynamic evaluation of neural sequence models.
\newblock \emph{ICML}, 2018.

\bibitem[Liu et~al.(2018{\natexlab{a}})Liu, Simonyan, and Yang]{liu2018darts}
Hanxiao Liu, Karen Simonyan, and Yiming Yang.
\newblock Darts: Differentiable architecture search.
\newblock \emph{arXiv preprint arXiv:1806.09055}, 2018{\natexlab{a}}.

\bibitem[Liu et~al.(2016)Liu, Wen, Yu, and Yang]{liu2016large}
Weiyang Liu, Yandong Wen, Zhiding Yu, and Meng Yang.
\newblock Large-margin softmax loss for convolutional neural networks.
\newblock In \emph{ICML}, pp.\  507--516, 2016.

\bibitem[Liu et~al.(2017)Liu, Wen, Yu, Li, Raj, and Song]{liu2017sphereface}
Weiyang Liu, Yandong Wen, Zhiding Yu, Ming Li, Bhiksha Raj, and Le~Song.
\newblock Sphereface: Deep hypersphere embedding for face recognition.
\newblock In \emph{The IEEE Conference on Computer Vision and Pattern
  Recognition (CVPR)}, volume~1, pp.\ ~1, 2017.

\bibitem[Liu et~al.(2018{\natexlab{b}})Liu, Lin, Liu, Liu, Yu, Dai, and
  Song]{liu2018learning}
Weiyang Liu, Rongmei Lin, Zhen Liu, Lixin Liu, Zhiding Yu, Bo~Dai, and Le~Song.
\newblock Learning towards minimum hyperspherical energy.
\newblock \emph{NeurIPS}, 2018{\natexlab{b}}.

\bibitem[Luong et~al.(2015)Luong, Pham, and Manning]{luong2015effective}
Minh-Thang Luong, Hieu Pham, and Christopher~D Manning.
\newblock Effective approaches to attention-based neural machine translation.
\newblock \emph{Proceedings of the 2015 Conference on Empirical Methods in
  Natural Language Processing}, 2015.

\bibitem[Mandt et~al.(2017)Mandt, Hoffman, and Blei]{mandt2017stochastic}
Stephan Mandt, Matthew~D Hoffman, and David~M Blei.
\newblock Stochastic gradient descent as approximate {B}ayesian inference.
\newblock \emph{The Journal of Machine Learning Research}, 2017.

\bibitem[Marcus et~al.(1993)Marcus, Marcinkiewicz, and
  Santorini]{marcus1993building}
Mitchell~P Marcus, Mary~Ann Marcinkiewicz, and Beatrice Santorini.
\newblock Building a large annotated corpus of english: The penn treebank.
\newblock \emph{Computational linguistics}, 19\penalty0 (2):\penalty0 313--330,
  1993.

\bibitem[Maronna et~al.(2018)Maronna, Martin, Yohai, and
  Salibi{\'a}n-Barrera]{maronna2018robust}
Ricardo~A Maronna, R~Douglas Martin, Victor~J Yohai, and Mat{\'\i}as
  Salibi{\'a}n-Barrera.
\newblock \emph{Robust statistics: theory and methods (with R)}.
\newblock Wiley, 2018.

\bibitem[Merity et~al.(2017{\natexlab{a}})Merity, McCann, and
  Socher]{merity2017revisiting}
Stephen Merity, Bryan McCann, and Richard Socher.
\newblock Revisiting activation regularization for language rnns.
\newblock \emph{ICML}, 2017{\natexlab{a}}.

\bibitem[Merity et~al.(2017{\natexlab{b}})Merity, Xiong, Bradbury, and
  Socher]{merity2016pointer}
Stephen Merity, Caiming Xiong, James Bradbury, and Richard Socher.
\newblock Pointer sentinel mixture models.
\newblock \emph{ICLR}, 2017{\natexlab{b}}.

\bibitem[Merity et~al.(2018{\natexlab{a}})Merity, Keskar, and
  Socher]{merity2017regularizing}
Stephen Merity, Nitish~Shirish Keskar, and Richard Socher.
\newblock Regularizing and optimizing lstm language models.
\newblock \emph{ICLR}, 2018{\natexlab{a}}.

\bibitem[Merity et~al.(2018{\natexlab{b}})Merity, Keskar, and
  Socher]{merity2018analysis}
Stephen Merity, Nitish~Shirish Keskar, and Richard Socher.
\newblock An analysis of neural language modeling at multiple scales.
\newblock \emph{arXiv preprint arXiv:1803.08240}, 2018{\natexlab{b}}.

\bibitem[Mikolov et~al.(2010)Mikolov, Karafi{\'a}t, Burget, {\v{C}}ernock{\`y},
  and Khudanpur]{mikolov2010recurrent}
Tom{\'a}{\v{s}} Mikolov, Martin Karafi{\'a}t, Luk{\'a}{\v{s}} Burget, Jan
  {\v{C}}ernock{\`y}, and Sanjeev Khudanpur.
\newblock Recurrent neural network based language model.
\newblock In \emph{Eleventh Annual Conference of the International Speech
  Communication Association}, 2010.

\bibitem[Mu et~al.(2018)Mu, Bhat, and Viswanath]{mu2017all}
Jiaqi Mu, Suma Bhat, and Pramod Viswanath.
\newblock All-but-the-top: Simple and effective postprocessing for word
  representations.
\newblock \emph{ICLR}, 2018.

\bibitem[Press(2019)]{press2019partially}
Ofir Press.
\newblock Partially shuffling the training data to improve language models.
\newblock \emph{arXiv preprint arXiv:1903.04167}, 2019.

\bibitem[Press \& Wolf(2016)Press and Wolf]{press2016using}
Ofir Press and Lior Wolf.
\newblock Using the output embedding to improve language models.
\newblock \emph{Proceedings of the 15th Conference of the European Chapter of
  the Association for Computational Linguistics: Volume 2, Short Papers}, 2016.

\bibitem[Rae et~al.(2018)Rae, Dyer, Dayan, and Lillicrap]{rae2018fast}
Jack~W Rae, Chris Dyer, Peter Dayan, and Timothy~P Lillicrap.
\newblock Fast parametric learning with activation memorization.
\newblock \emph{ICML}, 2018.

\bibitem[Ranzato et~al.(2016)Ranzato, Chopra, Auli, and
  Zaremba]{ranzato2015sequence}
Marc'Aurelio Ranzato, Sumit Chopra, Michael Auli, and Wojciech Zaremba.
\newblock Sequence level training with recurrent neural networks.
\newblock \emph{ICLR}, 2016.

\bibitem[Sankaranarayanan et~al.(2018)Sankaranarayanan, Jain, Chellappa, and
  Lim]{sankaranarayanan2018regularizing}
Swami Sankaranarayanan, Arpit Jain, Rama Chellappa, and Ser~Nam Lim.
\newblock Regularizing deep networks using efficient layerwise adversarial
  training.
\newblock In \emph{Thirty-Second AAAI Conference on Artificial Intelligence},
  2018.

\bibitem[Sennrich et~al.(2016{\natexlab{a}})Sennrich, Haddow, and Birch]{BPE}
Rico Sennrich, Barry Haddow, and Alexandra Birch.
\newblock Neural machine translation of rare words with subword units.
\newblock In \emph{ACL}, 2016{\natexlab{a}}.

\bibitem[Sennrich et~al.(2016{\natexlab{b}})Sennrich, Haddow, and
  Birch]{sennrich2015neural}
Rico Sennrich, Barry Haddow, and Alexandra Birch.
\newblock Neural machine translation of rare words with subword units.
\newblock \emph{Proceedings of the 54th Annual Meeting of the Association for
  Computational Linguistics (Volume 1: Long Papers)}, 2016{\natexlab{b}}.

\bibitem[Srivastava et~al.(2014)Srivastava, Hinton, Krizhevsky, Sutskever, and
  Salakhutdinov]{srivastava2014dropout}
Nitish Srivastava, Geoffrey Hinton, Alex Krizhevsky, Ilya Sutskever, and Ruslan
  Salakhutdinov.
\newblock Dropout: a simple way to prevent neural networks from overfitting.
\newblock \emph{The Journal of Machine Learning Research}, 15\penalty0
  (1):\penalty0 1929--1958, 2014.

\bibitem[Szegedy et~al.(2013)Szegedy, Zaremba, Sutskever, Bruna, Erhan,
  Goodfellow, and Fergus]{szegedy2013intriguing}
Christian Szegedy, Wojciech Zaremba, Ilya Sutskever, Joan Bruna, Dumitru Erhan,
  Ian Goodfellow, and Rob Fergus.
\newblock Intriguing properties of neural networks.
\newblock \emph{arXiv preprint arXiv:1312.6199}, 2013.

\bibitem[Tsochantaridis et~al.(2005)Tsochantaridis, Joachims, Hofmann, and
  Altun]{tsochantaridis2005large}
Ioannis Tsochantaridis, Thorsten Joachims, Thomas Hofmann, and Yasemin Altun.
\newblock Large margin methods for structured and interdependent output
  variables.
\newblock \emph{Journal of machine learning research}, 6\penalty0
  (Sep):\penalty0 1453--1484, 2005.

\bibitem[Vaswani et~al.(2017)Vaswani, Shazeer, Parmar, Uszkoreit, Jones, Gomez,
  Kaiser, and Polosukhin]{vaswani2017attention}
Ashish Vaswani, Noam Shazeer, Niki Parmar, Jakob Uszkoreit, Llion Jones,
  Aidan~N Gomez, {\L}ukasz Kaiser, and Illia Polosukhin.
\newblock Attention is all you need.
\newblock In \emph{Advances in Neural Information Processing Systems}, pp.\
  6000--6010, 2017.

\bibitem[Vaswani et~al.(2018)Vaswani, Bengio, Brevdo, Chollet, Gomez, Gouws,
  Jones, Kaiser, Kalchbrenner, Parmar, Sepassi, Shazeer, and
  Uszkoreit]{tensor2tensor}
Ashish Vaswani, Samy Bengio, Eugene Brevdo, Francois Chollet, Aidan~N. Gomez,
  Stephan Gouws, Llion Jones, \L{}ukasz Kaiser, Nal Kalchbrenner, Niki Parmar,
  Ryan Sepassi, Noam Shazeer, and Jakob Uszkoreit.
\newblock Tensor2tensor for neural machine translation.
\newblock \emph{CoRR}, abs/1803.07416, 2018.
\newblock URL \url{http://arxiv.org/abs/1803.07416}.

\bibitem[Wan et~al.(2013)Wan, Zeiler, Zhang, Le~Cun, and
  Fergus]{wan2013regularization}
Li~Wan, Matthew Zeiler, Sixin Zhang, Yann Le~Cun, and Rob Fergus.
\newblock Regularization of neural networks using dropconnect.
\newblock In \emph{International Conference on Machine Learning}, pp.\
  1058--1066, 2013.

\bibitem[Wang et~al.(2018)Wang, Cheng, Liu, and Liu]{wang2018additive}
Feng Wang, Jian Cheng, Weiyang Liu, and Haijun Liu.
\newblock Additive margin softmax for face verification.
\newblock \emph{IEEE Signal Processing Letters}, 25\penalty0 (7):\penalty0
  926--930, 2018.

\bibitem[Wang et~al.(2019)Wang, Xia, He, Tian, Qin, Zhai, and
  Liu]{wang2018multiagent}
Yiren Wang, Yingce Xia, Tianyu He, Fei Tian, Tao Qin, ChengXiang Zhai, and
  Tie-Yan Liu.
\newblock Multi-agent dual learning.
\newblock In \emph{International Conference on Learning Representations}, 2019.

\bibitem[Weston et~al.(1999)Weston, Watkins, et~al.]{weston1999support}
Jason Weston, Chris Watkins, et~al.
\newblock Support vector machines for multi-class pattern recognition.
\newblock In \emph{Esann}, volume~99, pp.\  219--224, 1999.

\bibitem[Xie et~al.(2017)Xie, Liang, and Song]{xie2016diverse}
Bo~Xie, Yingyu Liang, and Le~Song.
\newblock Diverse neural network learns true target functions.
\newblock \emph{Artificial Intelligence and Statistics}, 2017.

\bibitem[Xu et~al.(2015)Xu, Ba, Kiros, Cho, Courville, Salakhudinov, Zemel, and
  Bengio]{xu2015show}
Kelvin Xu, Jimmy Ba, Ryan Kiros, Kyunghyun Cho, Aaron Courville, Ruslan
  Salakhudinov, Rich Zemel, and Yoshua Bengio.
\newblock Show, attend and tell: Neural image caption generation with visual
  attention.
\newblock In \emph{International conference on machine learning}, pp.\
  2048--2057, 2015.

\bibitem[Yang et~al.(2018)Yang, Dai, Salakhutdinov, and
  Cohen]{yang2017breaking}
Zhilin Yang, Zihang Dai, Ruslan Salakhutdinov, and William~W Cohen.
\newblock Breaking the softmax bottleneck: A high-rank rnn language model.
\newblock \emph{ICLR}, 2018.

\bibitem[Yu \& Deng(2016)Yu and Deng]{yu2016automatic}
Dong Yu and Li~Deng.
\newblock \emph{AUTOMATIC SPEECH RECOGNITION.}
\newblock Springer, 2016.

\bibitem[Zoph \& Le(2017)Zoph and Le]{zoph2016neural}
Barret Zoph and Quoc~V Le.
\newblock Neural architecture search with reinforcement learning.
\newblock \emph{ICLR}, 2017.

\end{thebibliography}

\end{document}